\documentclass{esannV2}
\pdfoutput=1
\usepackage{graphicx}
\usepackage[latin1]{inputenc}
\usepackage{amssymb,amsmath,amsthm,array}

\newcommand{\bp}{\mathbf{p}}

\newcommand{\bv}{\mathbf{v}}
\newcommand{\bP}{\mathbf{P}}
\newcommand{\bw}{\mathbf{w}}
\newcommand{\bx}{\mathbf{x}}
\newcommand{\hbx}{\hat{\mathbf{x}}}

\newcommand{\bX}{\mathbf{X}}
\newcommand{\hbX}{\hat{\mathbf{X}}}

\newcommand{\bW}{\mathbf{W}}
\newcommand{\bM}{\mathbf{M}}

\newcommand{\Real}{\mathbb{R}}
\newcommand{\trace}[1]{\mbox{Tr}\left[#1\right]}
\DeclareMathOperator{\st}{s.t.}
\DeclareMathOperator{\sign}{sign}
\newcommand{\rank}{\mbox{Rank}}
\newcommand{\bDelta}{\mathbf{\Delta}}

\newcommand{\E}[2]{\mathbb{E}_{#1}\left\{#2\right\}}

\newcommand{\bI}{\mathbf{I}}
\newcommand{\bzero}{\mathbf{0}}
\newcommand{\bone}{\mathbf{1}}
\newcommand{\sic}[1]{\mbox{SIC}\left(#1\right)}

\newtheorem{remark}{Remark}
\newtheorem{theorem}{Theorem}

\title{Informative Data Projections:\\A Framework and Two Examples}

\author{Tijl De Bie$^{1,2}$, Jefrey Lijffijt$^{1,2}$, Ra\'ul Santos-Rodriguez$^2$, and Bo Kang$^1$
\thanks{This work was supported by the European Union through the ERC Consolidator Grant FORSIED (project reference 615517).}
\vspace{.3cm}\\
1- Data Science Lab - Ghent University \\
Technicum, Sint-Pietersnieuwstraat 41, 9000 Gent - Belgium
\vspace{.1cm}\\
2- Dept. of Engineering Mathematics - University of Bristol \\
MVB Woodland Road, BS8 1UB, Bristol - United Kingdom\\
}

\begin{document}

\maketitle

\begin{abstract} 
Methods for Projection Pursuit aim to facilitate the visual exploration of high-dimensional data by identifying interesting low-dimensional projections. A major challenge is the design of a suitable quality metric of projections, commonly referred to as the \emph{projection index}, to be maximized by the Projection Pursuit algorithm. In this paper, we introduce a new information-theoretic strategy for tackling this problem, based on quantifying the amount of information the projection conveys to a user given their \emph{prior beliefs} about the data. The resulting projection index is a subjective quantity, explicitly dependent on the intended user. As a useful illustration, we developed this idea for two particular kinds of prior beliefs. The first kind leads to PCA (Principal Component Analysis), shining new light on when PCA is (not) appropriate. The second kind leads to a novel projection index, the maximization of which can be regarded as a robust variant of PCA. We show how this projection index, though non-convex, can be effectively maximized using a modified power method as well as using a semidefinite programming relaxation. The usefulness of this new projection index is demonstrated in comparative empirical experiments against PCA and a popular Projection Pursuit method.
\end{abstract}

\section{Introduction}

The analysis of high-dimensional data often starts with dimensionality reduction, to facilitate initial visual exploration by a human user. Most analysts will instinctively do this by Principal Component Analysis (PCA) \cite{jolliffe2002principal}: it is widely available, computationally efficient, easy to interpret, and in the common situation where the data lies close to a low-dimensional subspace, PCA is effective in retrieving it.
However, in user interactions with their PRIM-9 system for interactive data exploration \cite{friedman1974projection}, it was observed that the human operators tended to prefer projections that reveal \emph{some form of structure}, 
rather than projections of \emph{high variance} as preferred by PCA. Later \cite{huber1985projection} provided theoretical arguments for why projections in which the data are Normally distributed are \emph{least} interesting, as they essentially reveal no structure in the data.\footnote{This means that Independent Component Analysis (ICA) and Projection Pursuit (PP) are largely equivalent: appropriate PP methods can be and are being used to do ICA \cite{hyvarinen1998analysis,hyvarinen2004independent,hyvarinen1999fast}.}

Quantifying the precise extent to which a projection \emph{is} interesting, however, is riddled with conceptual and practical difficulties. In fact, it seemed obvious to the early PP research protagonists that a universally useful \emph{projection index} that formalizes the interestingness of a projection cannot exist (see e.g. \cite{huber1985projection}). The answer, therefore, was the introduction of lots of different projection indices. Most of these aim to quantify the extent to which the distribution of the projected data departs from the Normal distribution, and all strike a different balance between practical usefulness, computational complexity, and robustness against outliers (e.g. \cite{friedman1974projection,huber1985projection,hyvarinen1998analysis,friedman1987exploratory} and references therein). Indeed, due to the elusive nature of the core question of what makes a projection interesting to a given user, the focus shifted towards secondary questions around robustness aspects and computational properties of the projection indices.

\paragraph{Contributions in this paper} Here our aim is to return the focus to the user once again, and directly ask the question of how interesting a given data projection is \emph{to a particular user}. Our work presents the first generic design strategy for projection indices that explicitly depend on the intended user.

In Section \ref{sec:outline} we introduce a strategy for quantifying the interestingness of a projection as its information content against the background of a probability distribution representing the user's beliefs about the data. In Sections \ref{sec:pca} and \ref{sec:robust_pca} we then apply this strategy for two particular types of prior beliefs, leading to a novel interpretation for PCA in the first case, and a novel projection index in the second, the optimization of which represents a robust variant of PCA. Although this latter projection index is non-convex, we introduce two algorithms for effectively optimizing it. We end by empirically illustrating the benefits of this robust PCA variant as compared to standard PCA and FastICA, a popular PP method that is also used for ICA \cite{hyvarinen1999fast}.

\section{The subjective information content of a data projection -- general outline\label{sec:outline}}

The proposed strategy for quantifying the subjective information content of a data projection 
closely follows the generic approach introduced by \cite{de2011information,de2013subjective}, where the choices made
are extensively motivated. Here we will merely provide some intuition underlying the approach.

The strategy from \cite{de2011information} for quantifying information content rests on the availability of a representation of the user's
prior belief state in the form of a probability density $p_\bX$ over the set of possible values for the data $\bX$ -- \emph{in casu}
over the set $\Real^{n\times d}$. Given this so-called \emph{background distribution}, one can then compute the marginal probability
density function of a data projection $\bp_\bw=\bX\bw$ defined by the weight vector $\bw\in\Real^{d}$. 
We will denote this marginal probability density function as  $p_{\bX\bw}$.

We call a \emph{projection pattern} a statement of the form $\bp_\bw\in[\hbX\bw,\hbX\bw+\Delta\bone)$,
specifying that the value $\bp_\bw$ of the projected data lies within an interval of width $\Delta$ around $\hbX\bw$ (with $\hbX\in\Real^{n\times d}$ the empirical data). 
This is what is conveyed to a user through a scatter plot of the data projections $\hbX\bw$, with plotting resolution $\Delta$.
Clearly, the smaller the probability $\mbox{Prob}_{\bp_\bw\sim p_{\bX\bw}}\left(\bp_\bw\in[\hbX\bw,\hbX\bw+\Delta\bone)\right)$, the more the surprising and hence informative this pattern would be \emph{to that particular user}. This is argued more formally by \cite{de2011information}, where 
the negative logarithm of this probability is shown to be a good measure of Subjective Information Content (SIC). We denote this as follows:
\begin{equation*}
\sic{\hbX\bw}=-\log\left(\mbox{Prob}_{\bp_\bw\sim p_{\bX\bw}}\left(\bp_\bw\in[\hbX\bw,\hbX\bw+\Delta\bone)\right)\right)
\end{equation*}
\emph{This is what we propose as a generic projection index, quantifying the interestingness of a projection.}

An important question is how $p_\bX$ and hence the marginals $p_{\bX\bw}$ can be obtained, without overburdening the user.
In \cite{de2011information} it is suggested that the user is often capable of specifying
aspects of their belief state as constraints on expected values of specified statistics of the data.
He argued that the \emph{Maximum Entropy} (MaxEnt) distribution subject to these constraints is an attractive choice,
given its unbiasedness, its robustness, and in being an \emph{exponential family} model \cite{wainwright2008graphical},
the inference of which is well understood and often computationally tractable.

In Sections \ref{sec:pca} and \ref{sec:robust_pca}, we will develop this strategy for two different assumptions regarding the prior beliefs of the user, illustrating how it can lead to new algorithms, as well as to new insights into existing ones.
Throughout this paper, we assume the data has been centered (i.e. has zero mean).

\section{PCA: an information theoretic interpretation\label{sec:pca}}

Here we show how standard PCA can be derived using this generic strategy for designing projection indices. This exercise will also present new insight in what cases PCA is an effective PP approach. 

\subsection{The prior beliefs}

A user not expecting any outliers can be assumed capable of expressing an expectation about the value of the average two-norm squared of the data points:
\begin{equation}\label{eq:pca_belief}
\E{\bX\sim p_\bX}{\frac{1}{n}\sum_{i}^{n}\bx_i'\bx_i}=\sigma^2.
\end{equation}
I.e., the user has a specific expectation about the average squared norm of the data points.

The MaxEnt distribution 
subject to this constraint is well known and equal to a product distribution of multivariate Normal distributions $\mathcal{N}(\bzero,\sigma\bI)$, with one factor for each of the data points $\bx_i$.
More formally, the density function $p_\bX$ for the dataset representing the background distribution is:
\begin{equation}\label{eq:pca_bd}
p_\bX(\bX)=\prod_i p_\bx(\bx_i), \mbox{where } p_\bx \mbox{, defined as } 
p_\bx(\bx)=\frac{1}{\sqrt{2\pi\sigma^2}^{d}}\exp\left(-\frac{\bx'\bx}{2\sigma^2}\right),
\end{equation}
is the probability density function for each of the individual data points in the dataset.\footnote{
Note that it is not our intention to argue in favor of this; our intent is merely to investigate what is a suitable projection index \emph{if} this is an accurate representation of the prior belief state.}

\subsection{The subjective information content}

Given a Normal random vector $\bx\sim\mathcal{N}(\bzero,\sigma\bI)$, a projection onto a weight vector $\bw$ with $\bw'\bw=1$ is also Normal: $\bx'\bw\sim\mathcal{N}(0,\sigma)$. 
Thus, given the independence of the data points under the background distribution, 
the marginal probability density function $p_{\bX\bw}$ for the projection $\bp_\bw=\bX\bw$  of a dataset $\bX$ sampled from the background distribution is given by:
\begin{equation*}
p_{\bX\bw}(\bp_\bw)=\frac{1}{\sqrt{2\pi\sigma^2}^n}\exp\left(-\frac{\bp_\bw'\bp_\bw}{2\sigma^2}\right).
\end{equation*}
We can thus compute the SIC of a projection pattern $\bp_\bw\in[\hbX\bw,\hbX\bw+\Delta\bone)$ as minus the logarithm of its probability under this marginal density function $p_{\bX\bw}$. Noting that for small enough $\Delta$, $\mbox{Prob}_{\bp_\bw\sim p_{\bX\bw}}\left(\bp_\bw\in[\hbX\bw,\hbX\bw+\Delta\bone)\right)\approx \Delta^n\cdot p_{\bX\bw}(\hbX\bw)$, this leads to:
\begin{align}
\sic{\hbX\bw}&=-\log\left(p_{\bX\bw}(\hbX\bw)\right)-n\log(\Delta)\nonumber\\
&=\frac{n}{2}\log(2\pi\sigma^2)-n\log(\Delta) + \frac{1}{2\sigma^2}\bw'\hbX'\hbX\bw.\label{eq:pca_sic}
\end{align}

It is trivial to generalize this toward $r$-dimensional projections $\bP_{\bW_r}=\bX\bW_r$ of the dataset, defined by an orthogonal matrix $\bW_r\in\Real^{d\times r}$ with $\bW_r'\bW_r=\bI$. With $\bDelta'=\left(\begin{array}{cccc}\Delta_1&\Delta_2&\cdots&\Delta_r\end{array}\right)$ a vector containing the resolutions for each of the projections, then the information content of the pattern $\bP_{\bW_r}\in[\hbX\bW_r,\hbX\bW_r+\bone\bDelta')$ specifying $r$ projections simultaneously, is given by:
\begin{equation}
\sic{\hbX\bW_r}=\frac{nr}{2}\log(2\pi\sigma^2)-n\sum_{i=1}^r\log(\Delta_i) + \frac{1}{2\sigma^2}\trace{\bW_r'\hbX'\hbX\bW_r}.\label{eq:pca_sic2}
\end{equation}

\subsection{Finding the most informative projections}

For fixed $\Delta$, maximizing the SIC from Eq.~(\ref{eq:pca_sic}) is done by solving:
\begin{equation*}
\max_\bw \bw'\hbX'\hbX\bw, \st \bw'\bw=1,
\end{equation*}
equivalent to the optimization problem to be solved for finding the first principal component in classical PCA. Similarly, optimizing Eq.~(\ref{eq:pca_sic2}) is equivalent to finding the $r$ dominant PCA components.

\begin{remark}
The assumption that $\Delta$ is constant is not always warranted. When making a scatter plot, it is often desirable to stretch the axes in order to fill available space. Then $\Delta\propto\max(\hbX\bw)-\min(\hbX\bw)$, such that finding the most informative projection amounts to solving:
\begin{equation*}
\max_\bw\bw'\hbX'\hbX\bw - 2\sigma^2 n\log\left(\max(\hbX\bw)-\min(\hbX\bw)\right),\st \bw'\bw=1,
\end{equation*}
This is another explanation of why simple variance maximization done by PCA is rarely a good approach for exploratory data analysis.
\end{remark}

From Eq.~(\ref{eq:pca_sic2}) the most informative $r$-dimensional projection of the data are found by solving:
\begin{equation*}
\max_{\bW_r\in\Real^{d\times r}}\trace{\bW_r'\hbX'\hbX\bW_r}, \st \bW_r'\bW_r=\bI,
\end{equation*}
again equivalent to the optimization problem for finding the $r$ dominant principal components.


\begin{remark}
The exact value of $\sigma$ will have no effect on the relative information content of different possible projections, hence the user does not need to provide this value.
\end{remark}


\section{t-PCA: for users expecting a heavy tailed distribution\label{sec:robust_pca}}


The previous section elucidates the assumptions on the user (prior belief on average squared norm of the data points) and visualization approach (constant resolution) for PCA to be optimal. In the present section we will develop an alternative for PCA when the assumption on the user's prior beliefs is altered, to be more accommodating for outliers.

\subsection{The prior beliefs}

As the user's prior belief about the data, we here propose the following:
\begin{equation*}
\E{\mathbf{X}\sim p_\bX}{\frac{1}{n}\sum_{i}^{n}\log\left(1+\frac{1}{\rho}\bx_i'\bx_i\right)} = c.
\end{equation*}

Thus, rather than specifying an expectation on the spread of the data, for small values of $\rho$ the user specifies an expectation on the \emph{order of magnitude} of the spread of the data. When the user expects outliers to be present, they may feel able to specify an expectation on the average \emph{order of magnitude} of the 2-norms of the data points, rather than on the average of their 2-norms themselves.

For notational convenience, let us introduce the function
$\kappa(\nu)=\psi\left(\frac{\nu+d}{2}\right)-\psi\left(\frac{\nu}{2}\right)$,
where $\psi$ represents the digamma function. In the sequel the value of $\kappa^{-1}(c)$ will need to be used, denoted as $\nu$ for brevity. Then, the background distribution can be derived by relying on \cite{Zog:99}, where it is shown that the MaxEnt distribution subject to the specified prior information is the product of independent multivariate standard $t$-distributions with density function $p_\bx$ defined as:
\begin{equation*}
p_\bx(\bx) = \frac{\Gamma\left(\frac{\nu+d}{2}\right)}{\sqrt{(\pi\rho)^d}\Gamma\left(\frac{\nu}{2}\right)}\cdot\frac{1}{\left(1+\frac{1}{\rho}\bx'\bx\right)^{\frac{\nu+d}{2}}},
\end{equation*}
with a factor in this product distribution for each data point. Here $\Gamma$ represents the gamma function.

Note that for $\rho,\nu\rightarrow\infty,\frac{\rho}{\nu}\rightarrow\sigma^2$ this density function tends to the multivariate Normal density function with mean $\bzero$ and covariance $\sigma^2\bI$. For $\rho=\nu=1$ it is a multivariate standard Cauchy distribution, which is so heavy-tailed that its mean is undefined and its second moment is infinitely large. Thus, this type of prior belief can clearly model the expectation of outliers to varying degrees.

\subsection{The subjective information content}

The marginals of a $t$-distribution with given correlation matrix are again a $t$-distribution with the same number of degrees of freedom, obtained by simply selecting the relevant part of the correlation matrix \cite{KoN:04,Rot:13}. This means that the marginal density function for the data projections $\bp_\bw=\bX\bw$ onto a vector $\bw$ with $\bw'\bw=1$ (and $p_{\bw,i}\triangleq \bx_i'\bw$) is:
\begin{equation*}
p_{\bX\bw}(\bp_\bw)=\prod_i p_{\bx'\bw}(p_{\bw,i}),  \mbox{where }  p_{\bx'\bw}(p_{\bw,i})=\frac{\Gamma\left(\frac{\nu+1}{2}\right)}{\sqrt{\pi\rho}\Gamma\left(\frac{\nu}{2}\right)}\cdot\frac{1}{\left(1+\frac{1}{\rho}p_{\bw,i}^2\right)^{\frac{\nu+1}{2}}}.
\end{equation*}
Thus the SIC of the projection pattern $\bp_\bw\in[\hbX\bw,\hbX\bw+\Delta\bone)$ is:
\begin{equation}\label{eq:sic_robust_pca}
\sic{\hbX\bw}=\frac{\nu+1}{2}\sum_{i=1}^{n}\log\left(1+\frac{1}{\rho}(\hbx_i'\bw)^2\right)-n\log(\Delta)+\mbox{a constant}.
\end{equation}

This derivation can be generalized towards the information content of an $r$-dimensional projection onto the columns of an orthogonal matrix $\bW_r\in\Real^{d\times r}$ (i.e. $\bW_r'\bW_r=\bI$). Without proof, we state the information content of the pattern $\bP_{\bW_r}\in[\hbX\bW_r,\hbX\bW_r+\bone\bDelta')$:
\begin{equation}
\sic{\hbX\bW_r}\label{eq:sic_robust_pca2}
=\frac{\nu+r}{2}\sum_{i=1}^{n}\log\left(1+\frac{1}{\rho}\hbx_i'\bW_r\bW_r'\hbx_i\right)-n\sum_{i=1}^r\log(\Delta_i)+\mbox{a constant}.
\end{equation}

\subsection{Finding the most informative projections: an analysis and two algorithms}

As in the derivation of PCA, we will assume that $\bDelta$ is constant. 
For ease of exposition, we will focus on the SIC of a single projection as given by Eq.~(\ref{eq:sic_robust_pca}). Taking into account that $\bw'\bw=1$, and ignoring some constant factors and terms, maximising the SIC is thus equivalent to solving the following problem:
\begin{equation}\label{eq:opt_robust_pca}
\max_{\bw} \sum_{i=1}^{n} \log\left(\rho+(\hbx_i'\bw)^2\right), \st \bw'\bw=1.
\end{equation}
Clearly, the larger $\bw'\bw$, the larger the objective, so the constraint can be relaxed to $\bw'\bw\leq 1$.

\begin{remark}
Given the reliance of this approach on the multivariate $t$-distribution as a background distribution, we will refer to this approach as \emph{t-PCA}.
\end{remark}

\begin{remark}
Just like in PCA where the value of $\sigma$ has no effect on which pattern is most interesting, here the value of $\nu$ and thus of $c$ is absent from the final optimization problem, and thus it has no effect on which projection is the most interesting one. (Though $\sigma$ and $c$ do affect the value of the interestingness.) This significantly reduces the demands on the user in specifying their prior beliefs.
\end{remark}



\begin{remark}
By varying $\rho$, t-PCA interpolates between maximizing the \emph{arithmetic mean}, like PCA does, and maximizing the \emph{geometric mean} of the squares of the data projections, which is more robust against outliers.
Indeed, for $\rho=0$, the objective function is monotonically related to the geometric mean of the squares of the data projections $(\hbx_i\bw)^2$:
\begin{equation*}
\exp\left[\frac{1}{n}\sum_{i=1}^n\log(\hbx_i'\bw)^2\right] = \left(\prod_{i=1}^n (\hbx_i'\bw)^2\right)^{\frac{1}{n}}.
\end{equation*}
On the other hand, for $\rho\rightarrow\infty$, the objective function is monotonically related to the arithmetic mean, and thus becomes equivalent to the PCA objective function:
\begin{equation*}
\lim_{\rho\rightarrow\infty} \frac{\rho}{n}\sum_{i=1}^n\log\left(\rho+(\hbx_i'\bw)^2\right)-\rho\log(\rho) = \frac{1}{n}\sum_{i=1}^n(\hbx_i'\bw)^2.
\end{equation*}
\end{remark}

\subsubsection{The complexity of the optimization problem}

To get some insight into the computational complexity of problem~\ref{eq:opt_robust_pca}, let us consider the special case of $\rho=0$. The constraint $\bw'\bw\leq 1$ is convex, and the objective is concave as long as $\bw\in\mathcal{W}_\mathbf{s}\triangleq\{\bw|\sign(\hbX\bw)=\mathbf{s}\}$ for some fixed sign vector $\mathbf{s}$. Indeed, for $\rho=0$ the objective function can be rewritten as $\sum_{i=1}^n\log((\hbx_i'\bw)^2)=\sum_{i=1}^n\log\mbox{det}\left(\begin{array}{cc}s_i\hbx_i'\bw&0\\0&s_i\hbx_i'\bw\end{array}\right)$, which is the sum of $n$ (concave) log determinant functions of linear matrix functions of the parameters $\bw$.

This seems to suggest a possible solution strategy, at least for the case $\rho=0$: enumerate all possible sign vectors $\mathbf{s}$ for the dataset $\hbX$, find a weight vector $\bw$ for each of these, and locally optimize it using a convex optimization problem. However, according to Cover's Function-Counting Theorem \cite{cover1965geometrical}, the number of homogeneously linearly separable dichotomies of $n$ points in $d$-dimensional Euclidean space is $2\sum_{k=0}^{d-1} \binom{n-1}{k} = O\left((n-1)^{d-1}\right)$, which is clearly impractical.

Since even for the special case of $\rho=0$ it is impractical to maximize the information content exactly, we developed two approximation algorithms: the first one a modification of the power method for solving eigenvalue problems, and the second one a convex relaxation to a log-determinant optimization problem.

\subsubsection{A modified power method}

The stationarity condition of the Karush-Kuhn-Tucker (KKT) optimality conditions for Eq.~\eqref{eq:opt_robust_pca} is:
\begin{equation*}
\left(\sum_{i=1}^n \frac{\hbx_i\hbx_i'}{\rho+(\hbx_i'\bw)^2}\right)\bw = \lambda\bw,
\end{equation*}
with $\lambda\geq 0$ a KKT multiplier corresponding to the constraint $\bw'\bw\leq 1$.
Note that the matrix on the left hand side is essentially a weighted empirical covariance matrix for the data, where points contribute more if they have a smaller value for $(\hbx_i'\bw)^2$: the weight for $\hbx_i\hbx_i'$ is $\frac{1}{\rho+(\hbx_i'\bw)^2}$.

As pointed out above, this optimisation problem is not convex, and also the optimality conditions do not admit a closed form solution in terms of e.g. an eigenvalue problem. Instead we investigated the use of a simple gradient descent method with after each gradient step a projection onto the feasible set $\bw'\bw=1$. This can be formulated as a modified power method \cite{golub2012matrix}:
\begin{enumerate}
\item Start with an initial value for $\bw^{(0)}$, normalized to unit norm.
\item Iterate from $k=1$ until convergence or maximum number of iterations reached:
\begin{enumerate}
\item $\bv^{(k)}=\left(\bI+\alpha\sum_{i=1}^n \frac{\hbx_i\hbx_i'}{\rho+{(\hbx_i'\bw^{(k-1)})^2}}\right)\bw^{(k-1)}$.
\item $\bw^{(k)}=\frac{\bv^{(k)}}{\|\bv^{(k)}\|}$.
\end{enumerate}
\end{enumerate}
Here, $\alpha$ is a step-size parameter that controls the speed of convergence. Clearly this algorithm is not guaranteed to converge, but for sufficiently small $\alpha$ it does converge to a local optimum in practice.\footnote{A detailed convergence analysis is left as further work.}

For $\bw^{(0)}$, we use the dominant eigenvector of $\left(\sum_{i=1}^n\frac{\hbx_i\hbx_i'}{\rho+\hbx_i'\hbx_i}\right)$, which amounts to maximizing an approximation of the SIC obtained by approximating $(\hbx_i'\bw)^2$ with $\hbx_i'\bx_i$. 


As a first approximation, we search for good subsequent projections simply by iteratively deflating the data (i.e. projecting it onto the orthogonal complement of the previously found projections).

\subsubsection{A convex relaxation}

The problem can also be relaxed to a semidefinite log determinant optimization problem. We will do this more generally for finding the most informative $r$-dimensional projection for general $r$, i.e. of the optimization problem maximizing Eq.~(\ref{eq:sic_robust_pca2}) subject to $\bW_r'\bW_r=\bI$. After removing irrelevant constant terms and factors, the optimization problem we need to solve is equivalent with:
\begin{equation}\label{eq:1}
\max_{\bW_r\in\Real^{d\times r}} \sum_{i=1}^n\log\left(\rho+\bx_i'\bW_r\bW_r'\bx_i\right),
\st \bW_r'\bW_r=\bI.
\end{equation}
We claim that this problem can be rewritten in terms of a new variable $\bM\triangleq\bW_r\bW_r'$ as follows:
\begin{equation}\label{eq:2}
\max_{\bM\in\Real^{d\times d}} \sum_{i=1}^n\log\left(\rho+\bx_i'\bM\bx_i\right),
\st \left\{\begin{array}{l}\trace{\bM}=r,\\
\bM\succeq \bzero,\\
\bI-\bM\succeq \bzero,\\
\rank(\bM)=r.
\end{array}\right.
\end{equation}
\begin{theorem}
Problems \ref{eq:1} and \ref{eq:2} are equivalent.
\end{theorem}
\begin{proof}
Given $\bM=\bW_r\bW_r'$, the objective functions are clearly equivalent. Thus all we need to show is that the feasible sets are identical as well. It is easy to verify that the constraints in problem~\ref{eq:1} imply the constraints in problem~\ref{eq:2}. Also the converse is true. Indeed, the rank constraint forces all but $r$ eigenvalues to be $0$. The constraints $\bI-\bM\succeq\bzero$ and $\bM\succeq\bzero$ ensure all eigenvalues lie between $0$ and $1$. And the trace constraint ensures all eigenvalues add up to $r$. Hence, $\bM$ must have $r$ eigenvalues equal to $1$, and $d-r$ equal to $0$. Thus, we can write the eigenvalue decomposition of $\bM$ as: $\bM=\bW_r\bI\bW_r'$, with $\bW_r$ an orthogonal matrix as required by the constraint in the previous formulation. Thus, the constraint set of problem~\ref{eq:2} also implies the constraint of problem~\ref{eq:1}.
\end{proof}

The only non-convex constraint in problem formulation~\ref{eq:2} is the rank constraint. We suggest to drop that constraint to relax this problem. Note that for $r=1$, the constraint $\bI-\bM\succeq\bzero$ is redundant given $\trace{\bM}=1$ and $\bM\succeq \bzero$, and can therefore also be dropped. To obtain an estimate for the unrelaxed weight matrix $\bW_r$, we suggest to use the $r$ dominant eigenvectors of $\bM$.

\section{Empirical evaluation\label{sec:empirical}}

In Section \ref{sec:synthetic_data} we evaluate the behaviour of t-PCA in comparison with PCA in a controlled setting on synthetic data. Then, in Section \ref{sec:real_data}, we demonstrate the practical usefulness of t-PCA on some real-life datasets, and compare its results with PCA as well as a popular projection pursuit method (FastICA).

\subsection{Experiments on synthetic data\label{sec:synthetic_data}}

\paragraph{Comparing PCA with t-PCA}
We generated a synthetic dataset consisting of two populations: a population with a small spread and $8000$ data points, and a population with a large spread and $2000$ data points. More specifically, both populations were sampled from a $100$-dimensional multivariate Normal distribution with diagonal covariance. For the large population the variances were sampled from a $\chi^2$ distribution with 1 degree of freedom; for the small population the same process was used, but the covariance matrix was then multiplied by $100$. 

As shown in Fig.~\ref{fig:artificial_combined} (left figure), due to the sensitivity of PCA to outliers, the dominant PCA direction is determined almost exclusively by the small population with large spread. t-PCA (here with the power method), however, offers an insight into the large population with lower spread as well.

\begin{figure}[t]
\centering
\includegraphics[width=0.336\textwidth]{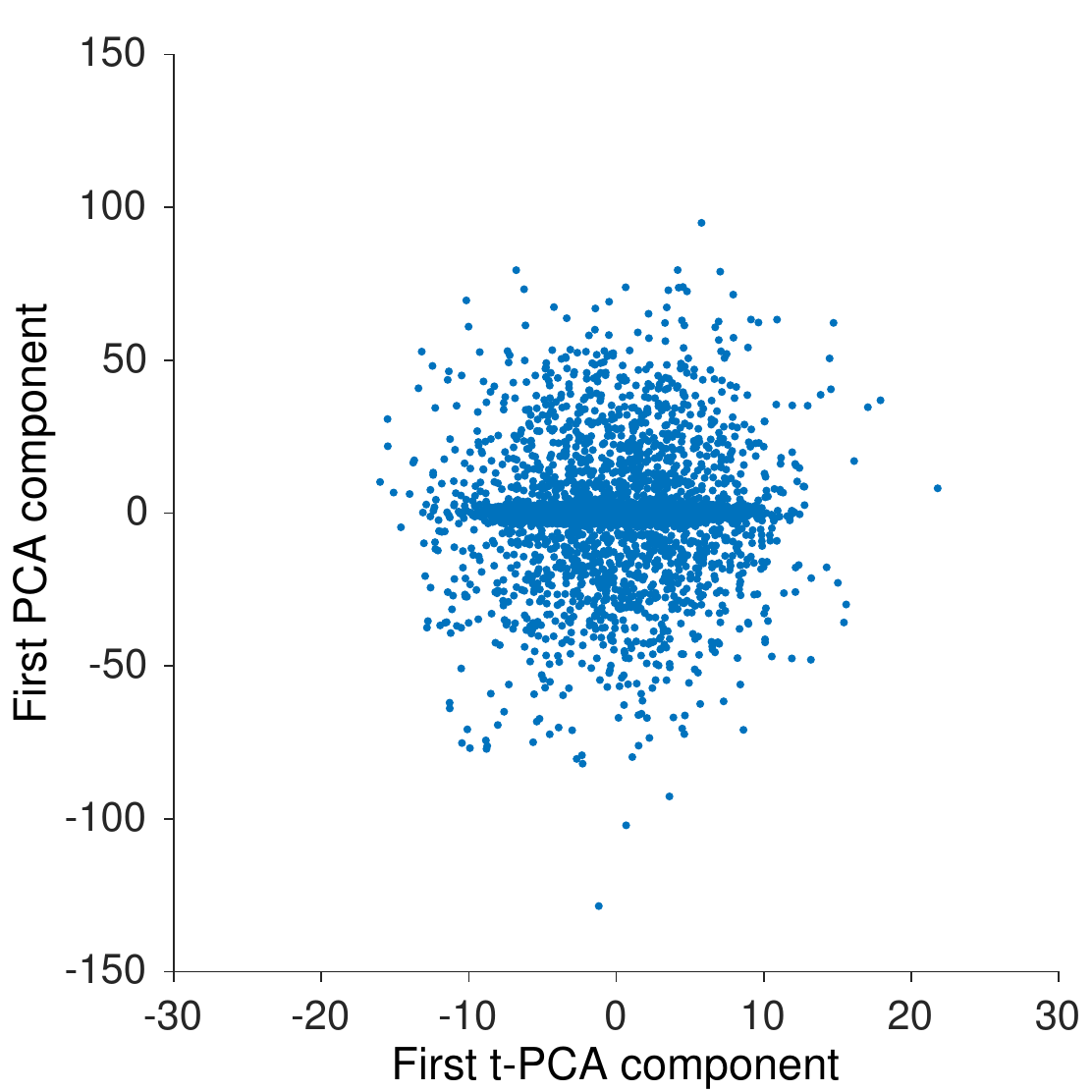}
\includegraphics[width=0.63\textwidth]{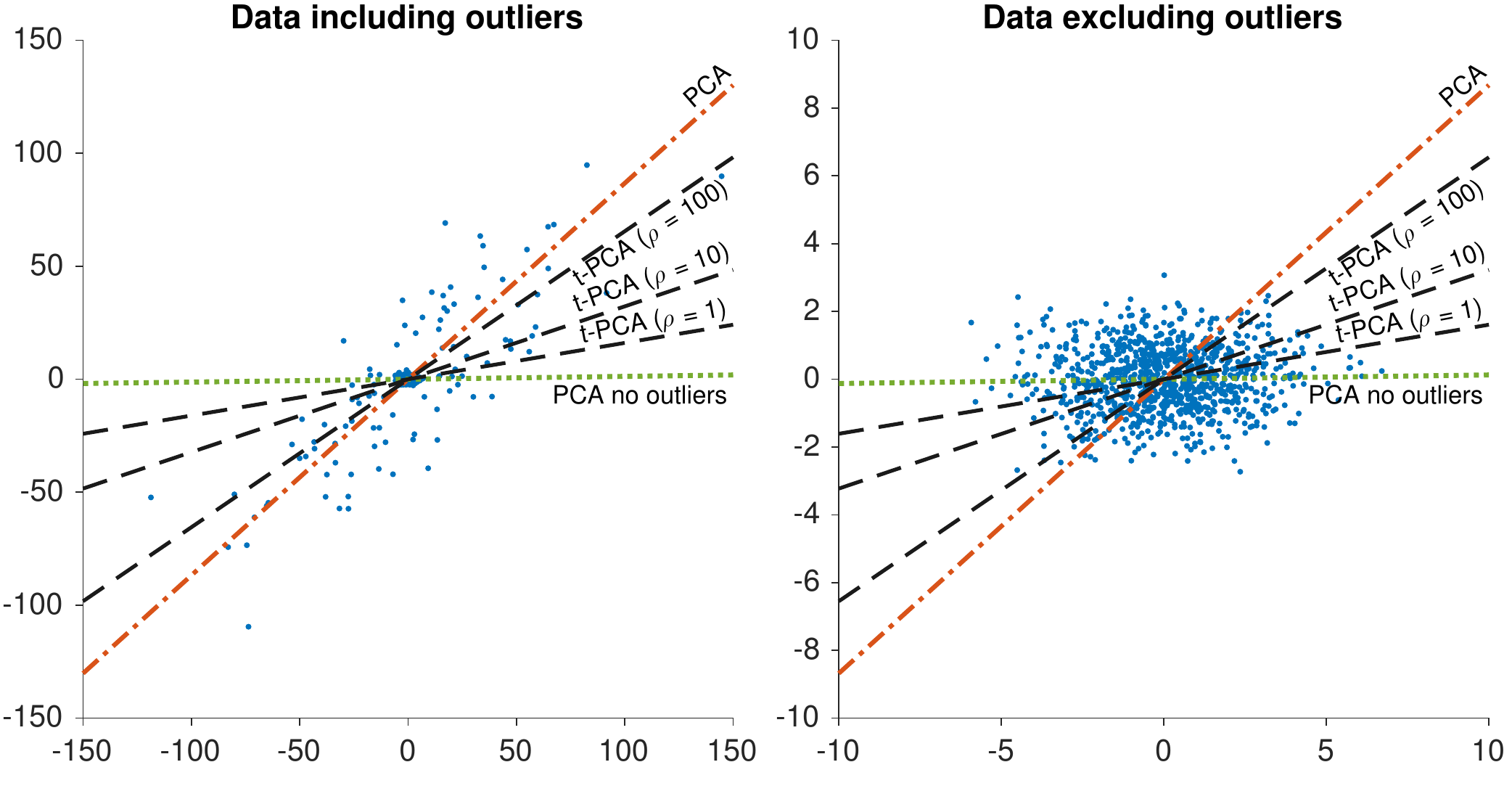}
\caption{Left: comparison of the dominant PCA projection (vertical axis) with the most informative t-PCA projection with $\rho=0$ (horizontal axis). t-PCA shows more detail in the central less spread out point cloud that contains most data points. Middle and right: a scatter plot of all data points in the original space, including (middle) and excluding (right) outliers, with weight vectors of PCA (dot-dashed red line), as well as t-PCA computed with the modified power method with $\rho=1,10,100$ (dashed black lines) and PCA fitted on data excluding the 100 outliers (dotted green line).\label{fig:artificial_combined}}
\end{figure}

The information content of the solution found by the SDP relaxation and modified power method are $1.43\times 10^4$ and $1.44\times 10^4$ respectively, while that of PCA is much worse at $3.18\times 10^3$. Note that the SDP relaxation took around 3 hours, compared to $30$ seconds for the modified power method, on this $10.000\times 100$ dataset. Usefully, the relaxation also provides us with an upper bound, namely $1.681\times 10^4$. Thus, both the SDP relaxation and the modified power method are close to optimal.

\paragraph{The effect of $\rho$}
To illustrate the robustness of t-PCA, consider a dataset consisting of two populations with different covariance structures: 1000 data points sampled from the Normal distribution $\mathcal{N}\left(\bzero,\left(\begin{array}{cc} 4 & 0 \\ 0 & 1 \end{array}\right)\right)$, and 100 `outliers' from a Normal distribution $\mathcal{N}\left(\bzero,\left(\begin{array}{cc} 16 & 12 \\ 12 & 13 \end{array}\right)\right)$.

The weight vector resulting from PCA is shown with a dashed red line in Fig.~\ref{fig:artificial_combined} (right two figures). The full black lines show the weight vectors retrieved by t-PCA, with values for $\rho$ equal to $1, 10,$ and $100$. The largest value of these resulted in the line closest to the PCA result. The green dotted line shows the weight vector that would have been found using PCA had there been no outliers at all (i.e. computed just on the first 1000 data points).
The middle plot in Fig.~\ref{fig:artificial_combined} demonstrates that the PCA result is determined primarily by the outliers. The right plot shows the same resulting weight vectors on top of a scatter plot of excluding the $100$ outliers, showing that t-PCA is hardly affected by the outliers.

\subsection{Experiments on real-life data\label{sec:real_data}}

We used two realistic datasets: the Shuttle Dataset\footnote{\texttt{https://archive.ics.uci.edu/ml/datasets/Statlog+(Shuttle)}} ($58000$ datapoints and 9 numerical dimensions) available from the UCI repository, and a reduced version\footnote{http://cs.nyu.edu/~roweis/data.html} of the 20 NewsGroups dataset ($16242$ datapoints and 100 dimensions). Both these datasets exhibit some complex structure, and the former in particular has a highly imbalanced cluster structure (class 1 contains $80\%$ of all data points).

The algorithms evaluated include PCA, t-PCA with $\rho$ equal to $10^{-5}$ multiplied by a measure of the scale of the data equal to the square root of the average squared norm of all data points, and a popular PP method often used for ICA, known as FastICA, with default parameters.

The results are shown in Fig.~\ref{fig:real}. The four top-level newsgroup classes in the reduced 20 NewsGroups dataset, and the seven classes in the Shuttle dataset, are shown in different colours. In all cases, the t-PCA version appears to reveal a more interesting structure in the data than either PCA or FastICA 
 do. Remarkably, for the 20 NewsGroups dataset, the t-PCA weight vectors are close to sparse: more than $97\%$ of the total variance of the weight vectors are due to just three, four of the 100 dimensions (i.e. words), respectively: `email', `help', and `problem' for the first projection and `case', `fact', `god', and `question' for the second. The FastICA weight vectors, in contrast, have almost all weight on a single dimension, explaining that all datapoints are projected onto one of just three points when projecting onto the two top ICA components.

\begin{figure}[t]
\centering
\includegraphics[width=\textwidth]{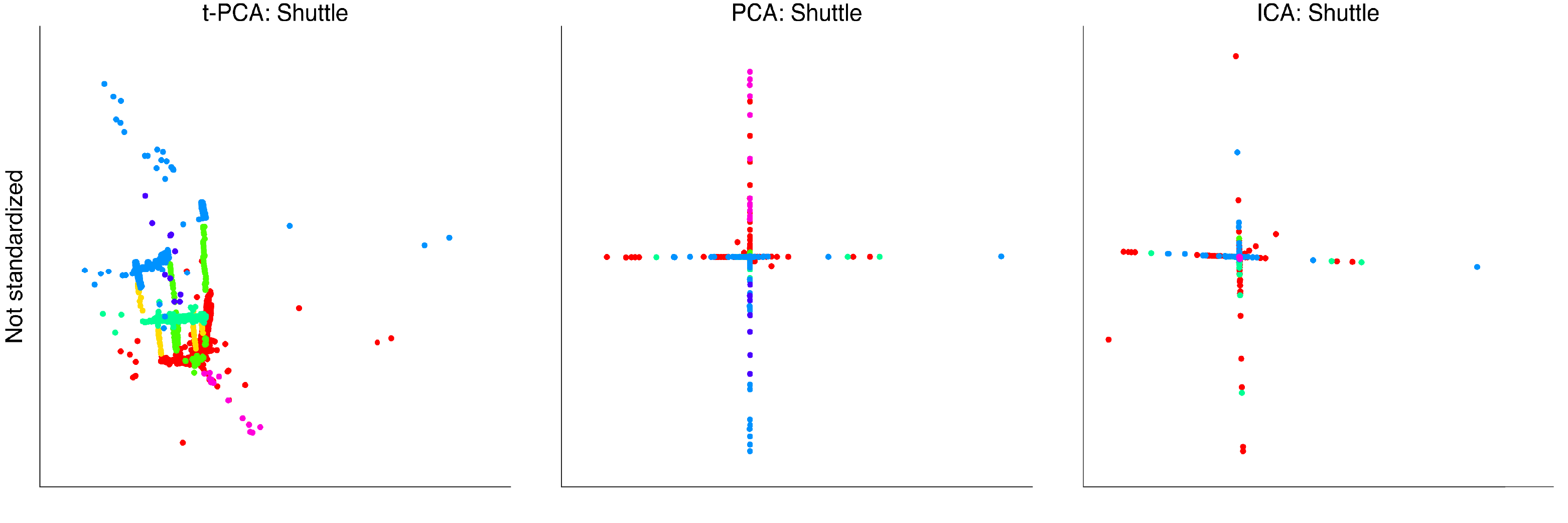}
\includegraphics[width=\textwidth]{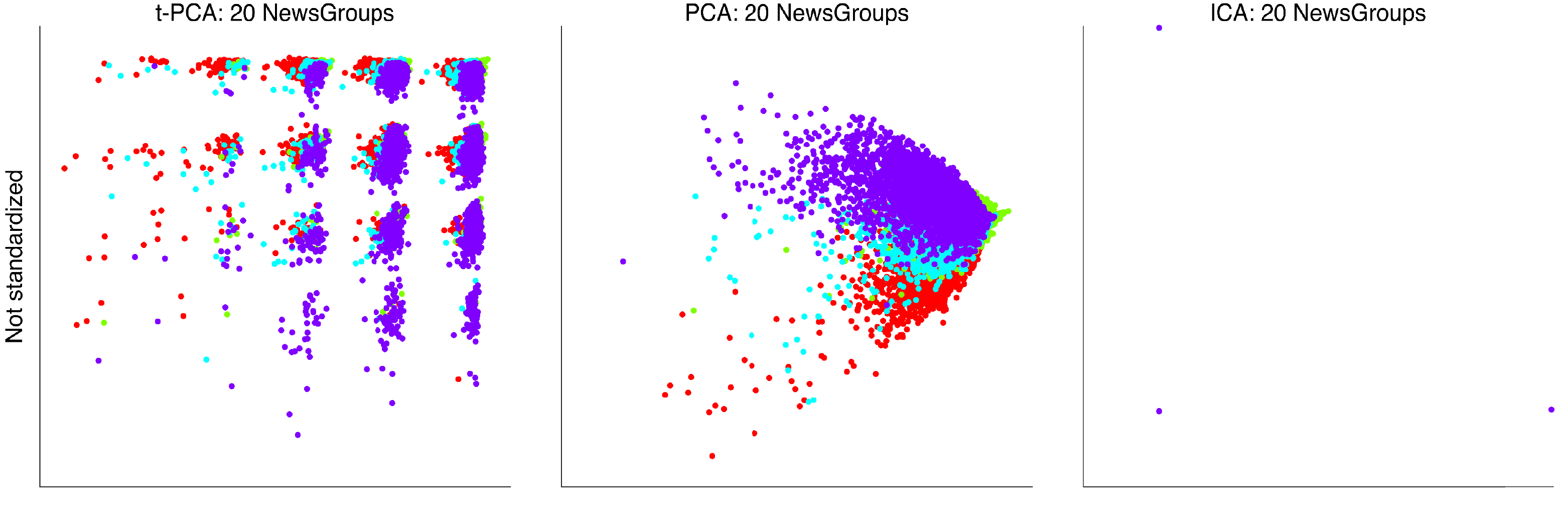}
\caption{The top 2 projections found by t-PCA (left), PCA (middle), and FastICA (right). Top row: Shuttle; bottom row: 20 NewsGroups.\label{fig:real}}
\end{figure}

\section{Conclusions}

PCA is often notoriously inappropriate for dimensionality reduction, e.g. in the presence of outliers. To address this the Projection Pursuit literature has introduced numerous \emph{projection indices} that quantify the interestingness of a projection in different ways. More recently, various authors also proposed principled \emph{robust} versions of PCA as an alternative, e.g. \cite{CLM:11,feng2012robust}. Yet, while these lines of work are useful when the assumptions made are valid, they do not fundamentally address how interesting a data projection is to a user. We presented a new approach to this elusive problem, explicitly recognizing the subjective nature of the notion `interestingness'.

Avenues for further work include alternative prior beliefs and data types, e.g. the case where the data is not real-valued but positive integer-valued or where the user assumes dependencies between the data points (when they are vectors in a time series, geographical locations, people in a social network, etc.). More immediately, the computational properties of the t-PCA optimization problem and its convex relaxation are worth investigating. Finally, an open question is to what extent the proposed strategy can be applied to non-linear dimensionality reduction as well.

\begin{footnotesize}
\bibliographystyle{unsrt}
\bibliography{informative_projections}
\end{footnotesize}

\end{document}